\documentclass[twocolumn]{article}

\usepackage[totalwidth=480pt,totalheight=680pt,margin=66pt]{geometry}

\usepackage{microtype}
\usepackage{graphicx}

\usepackage{amsmath}
\usepackage{amsthm}
\usepackage{amsfonts}
\usepackage{mathtools} % for the ceil function
\usepackage{amssymb}

\usepackage{subfigure}
\usepackage{booktabs} % for professional tables

\usepackage{algorithm}
\usepackage{algorithmic}
\newtheorem{innercustomthm}{Theorem}
\newenvironment{customthm}[1]
  {\renewcommand\theinnercustomthm{#1}\innercustomthm}
  {\endinnercustomthm}
  
\newtheorem{innercustomprop}{Proposition}
\newenvironment{customprop}[1]
 {\renewcommand\theinnercustomprop{#1} \innercustomprop}
{\endinnercustomprop}

\newtheorem{innercustomcor}{Corollary}
\newenvironment{customcor}[1]
 {\renewcommand\theinnercustomcor{#1} \innercustomcor}
{\endinnercustomcor}

\newtheorem{theorem}{Theorem}
\newtheorem{proposition}{Proposition}
\newtheorem{assumption}{Assumption}
\newtheorem{corollary}{Corollary}
\newtheorem{lemma}{Lemma}

\mathchardef\mhyphen="2D
\newcommand{\X}{m}
\newcommand{\nSW}{\tt SW}
\newcommand{\nD}{\tt D}
\newcommand{\SW}{\tt SW \mhyphen GLUCB}
\newcommand{\SWLin}{\tt SW \mhyphen LinUCB}
\newcommand{\D}{\tt D \mhyphen GLUCB}
\newcommand{\DLin}{\tt D \mhyphen LinUCB}
\newcommand{\Lin}{\tt LinUCB}
\newcommand{\Log}{\tt LogisticUCB}

\DeclarePairedDelimiter{\ceil}{\lceil}{\rceil}
\DeclareMathOperator*{\argmax}{arg\,max}
\DeclareMathOperator*{\argmin}{arg\,min}

\usepackage{hyperref}
\usepackage{natbib}
\setcitestyle{round}
\renewcommand{\cite}{\citep}
\usepackage{times}

\title{Algorithms for Non-Stationary Generalized Linear Bandits}
\author{Yoan Russac$^1$, Olivier Capp\'e$^1$, Aur\'elien Garivier$^2$ \\
${}^1$ DI ENS, CNRS, Inria, ENS, Université PSL; ${}^2$ UMPA, CNRS, Inria, ENS Lyon}
\date{}

\begin{document}

\maketitle

\begin{abstract}
  The statistical framework of Generalized Linear Models (GLM) can be
  applied to sequential problems involving categorical or ordinal
  rewards associated, for instance, with clicks, likes or ratings. In
  the example of binary rewards, logistic regression is well-known to
  be preferable to the use of standard linear modeling. Previous works
  have shown how to deal with GLMs in contextual online learning with
  bandit feedback when the environment is assumed to be stationary. In
  this paper, we relax this latter assumption and propose two upper
  confidence bound based algorithms that make use of either a sliding
  window or a discounted maximum-likelihood estimator. We provide
  theoretical guarantees on the behavior of these algorithms for
  general context sequences and in the presence of abrupt
  changes. These results take the form of high probability upper
  bounds for the dynamic regret that are of order
  $d^{2/3} \Gamma_T^{1/3} T^{2/3}$, where $d, T$ and $ \Gamma_T$ are
  respectively the dimension of the unknown parameter, the number of
  rounds and the number of breakpoints up to time $T$. The empirical
  performance of the algorithms is illustrated in simulated
  environments.
\end{abstract}

\section{Introduction}
The multi-armed bandit model is a well-known abstraction of the
exploration-exploitation dilemma that occurs whenever predictions need
to be made while learning a parameter of interest. When contextual
information is available, a popular framework is the stochastic linear
model \cite{dani2008stochastic,li2010contextual,rusmevichientong2010linearly,abbasi2011improved},
where the reward observed at each round is a noisy version of a linear
combination of the contextual features that describe the selected
action.

More precisely, we assume that at time $t$ a set of contextual actions
$\mathcal{A}_t \subset \mathbb{R}^d$ is available. Based on previous
choices and rewards, the learner selects one of them and observes the
associated reward. The learner's goal is to maximize the accumulated
rewards. The particularity of the bandit setting is that the learner
does not know the reward she would have obtained by selecting another
action. In a recommendation setting, the contextualized actions may
for instance combine information about both the users and the products
to be recommended. By selecting the action $A_t$, a noisy version of
$A_t^{\top} \theta^{\star}$ is observed, where $\theta^{\star}$ is an
unknown parameter associated with the environment.

The Generalized Linear Model (GLM) setting \cite{filippi_GLM,
  li2017provably} extends this model by assuming that conditionally on
$A_t$, the learner observes a noisy version of
$\mu(A_t^{\top} \theta^{\star})$, where $\mu$ is a non-linear mapping,
referred to as the inverse link or mean function. More details on the
probabilistic structure of GLMs are given in Section
\ref{sec_setting}. A particular case of great practical interest
occurs when $\mu$ is the logistic function, which is the dominant
approach for regression modeling with binary outcomes.

The classical bandit framework assumes stationarity of the environment
parameter $\theta^{\star}$. This is clearly unrealistic in many
potential applications. In news recommendation for instance, as
considered by \cite{li2010contextual}, it has been consistently
observed that the intrinsic interest in news stories is a decreasing
function of time to original publication date. But, on the other hand,
infrequent increases in interest for older items can also be triggered
by the publication of fresh news. Regularly restarting the learning
algorithm is a (frequent) basic approach to mitigate this
issue. However, there is also a strong interest for developing bandit
approaches that are inherently robust to possible changes in the
environment. The aim of this work is to propose and analyze methods
that achieve this goal in contextual bandits based on GLMs (which we
shall refer to as "generalized linear bandits'').
   
\textbf{Related Work.} \quad  Two types of
   approaches are generally adopted to deal with non-stationarity. 
   The first one consists in detecting
    changes in distribution \cite{auer2018adaptively,
  besson2019generalized} and restarting the algorithm
   whenever a change is
detected. The second one builds progressively forgetting policies
\cite{garivier2011upper} based either on the use of a sliding window --computing
 the estimator 
only on the most recent observations--, or, on the use of exponentially 
increasing weights to reduce the influence
 of past observations.
Both approaches have been studied in the
$K$-armed and linear settings. In linear bandits,
\citet{wu2018learning} build a pool of plausible models to make
recommendations. When no model satisfies a given statistical test, a
change point is declared and a new model is added to the pool. 
In
\cite{cheung2019learning} the sliding window approach is used to build
the least squares estimator. In
\cite{russac2019weighted} the past is progressively forgotten with
the use of a discount factor that gives more weights to recent
observations and the estimator is defined through weighted least squares.

Assessing the performance of these methods, requires quantified measures of
non-stationarity and here again there are several options.
The notion of \textit{variation budget} that
includes both slowly and abruptly changing environments was considered
in the $K$-armed bandit setting by \citet{besbes2014stochastic} and
in the linear setting by \citet{cheung2019learning}
for example.
In this work, as in --among others-- \cite{garivier2011upper, liu2018change,
cao2018nearly}, we focus on abruptly changing
environments, and measure non-stationarity by the number
of breakpoints up to time $T$.

GLMs with bandit feedback were studied by \citet{filippi_GLM}
with a fixed actions set; the authors proposed a first UCB algorithm in this
setting. Our work extends theirs to the case where the preference parameters
$\theta^{\star}$ can dynamically evolve over time. We stress that their analysis assumed 
static actions whereas ours also works with time dependent
actions sets. No regularization term was used in \citet{filippi_GLM} 
implying unsatisfactory initialization assumptions that we were 
able to remove by considering a penalized estimator.  

Another analysis 
was proposed by \citet{li2017provably}, where, in contrast to our work,
statistical assumptions are made 
on the distributions of the contextual vectors, allowing the
use of results from random matrix theory for establishing
concentration inequalities. We work in the more 
general framework where the available actions at each
 round can even be chosen by an adversary.
 
 Randomized algorithms have also been developed to study 
 generalized linear bandits. The extension of Thompson Sampling
 to this setting was analyzed by \citet{abeille2017linear} and
 a $O(d^{3/2} \sqrt{T \log(K)})$ regret bound valid for infinite
 actions sets was derived.
 In \cite{kveton2019randomized} two others randomized algorithms
 are proposed. One method consists in fitting a GLM on a randomly 
 perturbed history of the past rewards to guarantee sufficient exploration.
 The second method consists in sampling a GLM from the Laplace approximation
 to the posterior distribution. In a $d$-dimensional problem with $K$ fixed
 actions the $T$ rounds regret of those methods is of order $O(d\sqrt{\log(K)T})$.
 Both methods in \cite{kveton2019randomized} assume a static actions
  set and have a logarithmic dependence in the number of actions. In contrast, the upper-bounds
 on the regret that we obtain do not depend on the number of available actions.

The regret of most existing algorithms for generalized linear
bandits is inversely proportional to the minimum value of the derivative
of the inverse link function.  This quantity can be large (as in the
logistic model), and hence designing policies that do not depend on this
quantity is of particular interest. In the particular case of logistic
bandits under strong assumptions on the features and $\theta^{\star}$,
\citet{dong2019performance} propose a first Bayesian analysis that
does not depend on this quantity. However, the analysis of \citet{dong2019performance}
relies on specifics of the logistic model and cannot be directly
extended to the broader class of GLMs or to control the (stronger)
notion of frequentist regret.

Non-Stationary GLM have been studied in recent works that  consider both abruptly changing
 and smoothly changing environments \cite{cheung2019hedging, zhao2020simple}. However 
 the analysis in both of these works have gaps: \cite{zhao2020simple} define $c_{\mu}$ 
 (see our Assumption \ref{assumption_c_mu}) as the 
 minimum value of $\dot{\mu}(a^{\top} \theta)$ for $\theta \in \mathbb{R}^d$, which for 
 the logistic regression model would be zero; \cite{cheung2019hedging} implicitly 
 assume that the maximum likelihood estimator at all time instants belongs to 
 $\Theta = \{ \theta \in \mathbb{R}^d, 
 \Vert \theta \rVert_2 \leq  S \}$ , which may not be true in general.

\textbf{Main Contributions.} \quad In this paper, we propose the first
upper confidence bound algorithms designed for non-stationary
environments in generalized linear bandits. The algorithms are
extensions of the SW-LinUCB \cite{cheung2019learning} and the D-LinUCB
\cite{russac2019weighted} algorithms and can achieve a dynamic regret
over $T$ rounds of order $O(d^{2/3} \Gamma_T^{1/3} T^{2/3})$, where
$\Gamma_T$ denotes the number of breakpoints up to time $T$.  This
rate is known to be optimal up to logarithmic terms. We propose an
original and simplified analysis that is valid with time-dependent
actions sets and does not required statistical assumption on the
distribution of the contextual vectors.  In the two algorithms, we
make use of (possibly weighted) penalized maximum likelihood
estimation. To the best of our knowledge, the analysis of penalized
MLE in generalized linear bandits is also an original
contribution. Note that in non-stationary environments there is no
simple way to circumvent the need for regularization by using a proper
initialization for the algorithms (as is done by \citet{filippi_GLM}):
when using the sliding window for instance
the initialization procedure would need to be
repeated regularly, resulting in a large drop in performance.

\section{Problem Setting}
\label{sec_setting}

Extending the generalized linear bandit framework introduced by \citet{filippi_GLM}, 
we consider a structured bandit
model where the number of arms at each round is upper-bounded by a finite 
$K$, the action set $\mathcal{A}_t$ is time-dependent and at each step an action
$A_t \in \mathcal{A}_t \subset \mathbb{R}^d$ is chosen.
The conditional distribution of the rewards
belongs to a \textit{canonical exponential family} wrt a reference measure $\mu$:
$d\mathbb{P}_{A^{\top} \theta}(x) = d\mathbb{P}_{\theta}(x | A) = \exp(
x A^{\top} \theta - b(A^{\top} \theta) + c(x))d\mu(x)$, where $c(.)$ is a
real-valued function an $b(.)$ is assumed to be twice continuously
differentiable.
 
A random variable $X$ with the above density verifies
$\mathbb{E}(X) = \dot{b}(A^{\top} \theta)$ and
$\textnormal{var}(X) = \ddot{b}(A^{\top} \theta)$, showing that $b(.)$ is strictly convex. The inverse
link function is $\mu=\dot{b}$.

At time $t$, when the action $A_t$ is chosen, the received reward $X_t$ is
conditionally independent of the past actions and satisfies
$\mathbb{E}(X_t | A_t) = \mu (A_t^{\top} \theta^{\star})$. 
In the non-stationary framework, the 
difference is that at time $t$ the
conditional expectation is equal to $\mu(A_t^{\top} \theta^{\star}_t)$
rather than $\mu(A_t^{\top} \theta^{\star})$ after selecting an action $A_t$.
   
We first assume that the L2-norms of the available actions and the
admissible parameters $(\theta^{\star}_t)_{t \geq 1}$ are bounded,

\begin{assumption}
\label{assumption_actions}
 $\forall t \geq 1, \forall a \in \mathcal{A}_t, \,  \lVert a \rVert_2 \leq L$.
\end{assumption}

\begin{assumption}
\label{assumption_param}
 $\forall t \geq 1, \,  \lVert \theta^{\star}_t \rVert_2 \leq S$.
\end{assumption}
The following assumption is also useful to derive concentration bounds.
\begin{assumption}
\label{assumption_noise}
There exists $\X > 0$ such that for any $t \geq 1$, $0 < X_t < \X$.
\end{assumption}

\underline{Remark:} We define the noise term as
$\eta_t = X_t - \mu(A_t^{\top} \theta^{\star}_t)$, so that
$\mathbb{E} [\eta_t | X_{t-1},..., X_1] = 0$. As explained in Lemma
\ref{lemma:hoeffding_conditional} of Appendix~\ref{subsec:hoeffding_cond},
 $\eta_t$ is $\X/2$-subgaussian conditionally on the past.
 
The maximum likelihood estimator $\hat{\theta}_t$ based on the rewards
$X_1,...,X_{t-1}$ and the selected actions $A_1,...,A_{t-1}$ is
defined as the maximizer of
 \begin{equation} 
 \label{eq_log_likelihood}
 \sum_{s=1}^{t-1} \log( \mathbb{P}_{\theta}(X_s | A_s)) =  \sum_{s=1}^{t-1} X_s A_s^{\top} \theta - b(A_s^{\top} \theta) + c(X_s)\,.
 \end{equation}
By convexity of $b$, the rhs of the previous equation is concave in $\theta$. 
After differentiating the log-likelihood, $\hat{\theta}_t$ appears as the solution of the equation
\begin{equation}
\sum_{s=1}^{t-1} (X_s - \mu(A_s^{\top} \theta))A_s = 0\;.
\end{equation}
Extra assumptions on the link function are also necessary for the
theoretical analysis, in particular:
\begin{assumption}
\label{assumption_c_mu}
The inverse link function $\mu: \mathbb{R} \mapsto \mathbb{R}$ is
a continuously differentiable Lipschitz function, with Lipschitz constant $k_{\mu}$, such that
$$c_{\mu} = \inf_{\lVert \theta \rVert_2 \leq S, \lVert a \rVert_2 \leq
  L } \dot{\mu} (a^{\top} \theta) > 0\;.
$$
\end{assumption}

Assumption \ref{assumption_c_mu} could be relaxed by only considering
the $\theta$ parameters in a neighborhood of the true unknown
parameter $\theta^{\star}$ as in \cite{li2017provably}. However,
doing so would require assuming that the actions are drawn from a
distribution verifying particular conditions. In a non-stationary
environment, even this extra assumption is not always sufficient as
$\theta^{\star}$ evolves over time.

In the non-stationary environment, the goal of the learner is to
minimize the expected \textit{dynamic regret} defined as
 $$
 R_T = \sum_{t=1}^T \max_{a \in \mathcal{A}_t} \mu(a^{\top} \theta^{\star}_t) - \mu(A_t^{\top} \theta^{\star}_t)\;.
 $$

Note that in contrast to the settings considered by
\citet{filippi_GLM} or \citet{kveton2019randomized}, the available
actions sets $\mathcal{A}_t$ are time-dependent. Hence, in the above definition
of regret, the best action can differ between rounds and it is no
more possible to control the regret by upper-bounding the number of
times each sub-optimal arm is played.

\section{Algorithms}

In this section, we describe two estimators together with the corresponding algorithms. 
The first estimator is based on a sliding window
where only the $\tau$ most recent rewards and actions are
considered. The second one uses a discount factor $\gamma$ and gives more
weight to the most recent actions and rewards. Both estimators rely
on a penalization of the log-likelihood that has a regularizing effect
and avoids the need of specific initialization procedures.

\subsection{Sliding Window and Penalized MLE}
\label{subsection_SW}
The first estimator we consider is a truncated version of the
penalized MLE. Equation \eqref{eq_log_likelihood} is replaced by
\begin{equation}
\label{eq_log_likelihood_SW}
 \sum_{s=\max(t-\tau,1)}^{t-1} \log( \mathbb{P}_{\theta}(X_s | A_s)) - \frac{\lambda}{2} \lVert \theta\rVert_2^2\;.
\end{equation}

By differentiating the (strictly
concave) penalized log-likelihood, $\hat{\theta}_t^{\nSW}$ appears as the unique solution of
\begin{equation}
\label{eq_MLE_SW}
\sum_{s=\max(t-\tau,1)}^{t-1} (X_s - \mu(A_s^{\top} \theta)) A_s - \lambda \theta = 0\;.
\end{equation}

We introduce
\begin{equation}
\label{eq_Design_matrix_SW}
 V_{t-1} = \sum_{s=\max(1, t- \tau)}^{t-1} A_s A_s^{\top} + \frac{\lambda}{c_{\mu}} I_d\;,
 \end{equation}
and we define $g_{t}(\theta) = \sum_{s= \max(1,t + 1 -\tau)}^{t} \mu(A_s^{\top}
 \theta) A_s + \lambda \theta$ and $\tilde{\theta}^{\nSW}_t$ by

\begin{equation}
\label{eq:theta_tilde_SW}
\tilde{\theta}^{\nSW}_t  =  \argmin_{\lVert \theta \rVert_2 \leq S} \, \lVert g_{t-1}(\hat{\theta}_t^{\nSW}) - g_{t-1}(\theta)
 \rVert_{V_{t-1}^{-1}}\;,
\end{equation} 

where $V_{t-1}$ is defined in Equation \eqref{eq_Design_matrix_SW}. We
need to consider $\tilde{\theta}^{\nSW}_t$ because $\hat{\theta}_t$ is
not guaranteed to satisfy $\lVert \hat{\theta}_t \rVert_2 \leq
S$ and the lower bound on $\dot{\mu}$ with $c_{\mu}$ is only valid for parameters whose L2 norm 
is smaller than $S$.
 $\tilde{\theta}^{\nSW}_t$ should be understood as a "projection"
on the admissible parameters.

Using this notation, we can now present our first algorithm for
generalized linear bandits in non-stationary environments. $\SW$
(Sliding Window Generalized Linear Upper Confidence Bound) uses a
sliding window to focus on the most recent events. The $\SW$
algorithm uses a confidence bonus $\rho^{\nSW}$
that will defined in Section~\ref{subsec:analysis_SW} devoted to the
analysis of the algorithms (see Equation \eqref{eq_rho_SW} for the definition of  $\rho^{\nSW}$).

\begin{algorithm}[h]
\caption{$\SW$}
   \label{alg:S-GLM}
\begin{algorithmic}
   \STATE {\bfseries Input:} Probability $\delta$, dimension $d$, regularization $\lambda$, 
upper bound for actions $L$, upper bound for parameters $S$, sliding window $\tau$.
\STATE {\bfseries Initialize:} $V_0 = \lambda/{c_{\mu}} I_d$, $\hat{\theta}^{\nSW}_0 = 0_{\mathbb{R}^d}$.
   \FOR{$t=1$ {\bfseries to} $T$}
   \STATE Receive $\mathcal{A}_t$, compute $\hat{\theta}_t^{\nSW}$ according to (\ref{eq_MLE_SW})
   \STATE {\bfseries if} $\lVert \hat{\theta}_t^{\nSW} \rVert_2 \leq S$ {\bfseries let} $\widetilde{\theta}_t^{\nSW} =  \hat{\theta}_t^{\nSW}$
    {\bfseries else} compute
   $\tilde{\theta}_t^{\nSW}$ with (\ref{eq:theta_tilde_SW})
   \STATE {\bfseries Play} $A_t \hspace{-0.05cm}= \hspace{-0.05cm} \argmax_{a \in \mathcal{A}_t} \hspace{-0.05cm} \left( \mu(a^{\top} \widetilde{\theta}^{\nSW}_t) 
   \hspace{-0.05cm}+ \hspace{-0.05cm}
   \rho^{\nSW}_t(\delta) \lVert a \rVert_{V_{t-1 }^{-1}} \right)$
  \STATE {\bfseries Receive} reward $X_t$
  \STATE {\bfseries Update:} 
  \IF {$t \leq \tau$}
  \STATE $V_t \leftarrow V_{t-1} + A_t A_t^{\top}$
  \ELSE
  \STATE $V_t \leftarrow V_{t-1} + A_t A_t^{\top} - A_{t-\tau} A_{t-\tau} ^{\top}$
 \ENDIF  
   \ENDFOR
\end{algorithmic}
\end{algorithm}

\subsection{Discounting Factors and Penalized MLE}
\label{subsection_Discounted}
The second estimator we construct is based on a weighted penalized
log-likelihood. Rather than using Equation \eqref{eq_log_likelihood},
$\hat{\theta}_t^{\nD}$ is defined as the unique maximum of
\begin{equation}
\label{eq_log_likelihood_D}
\sum_{s=1}^{t-1} \gamma^{t-1-s} \log( \mathbb{P}_{\theta}(X_s | A_s)) - \frac{\lambda}{2} \lVert \theta\rVert_2^2\;. 
\end{equation}
As before, thanks to the concavity in $\theta$, $\hat{\theta}_t^{\nD}$
is also the solution of
\begin{equation}
\label{eq_MLE_D}
\sum_{s=1}^{t-1} \gamma^{t-1-s} (X_s - \mu(A_s^{\top} \theta)) A_s - \lambda \theta = 0\;.
\end{equation}

We introduce
\begin{equation}
\label{eq_Design_matrix_D}
 W_{t} = \sum_{s=1}^{t} \gamma^{t-s} A_s A_s^{\top} + \frac{\lambda}{c_{\mu}} I_d
 \end{equation}
  and  
  \begin{equation}
\label{eq_Design_matrix_D2}
 \widetilde{W}_{t} = \sum_{s=1}^{t} \gamma^{2(t-s)} A_s A_s^{\top} + \frac{\lambda}{c_{\mu}} I_d\;.
 \end{equation}
 
As in the linear setting, there is a need to introduce a covariance
matrix containing the squares of the weights because the stochastic term can
only be controlled in $\widetilde{W}_t^{-1}$ norm \cite{russac2019weighted}.
 
Let $g_{t}: \mathbb{R}^d \mapsto \mathbb{R}^d$ denote the following function
$$
g_{t}(\theta) = \sum_{s=1}^{t} \gamma^{t-s} \mu(A_s^{\top} \theta) A_s + \lambda \theta\;.
$$
Finally, let $\tilde{\theta}^{\nD}_t$ be defined as 
\begin{equation}
\label{eq:theta_tilde_D}
\tilde{\theta}^{\nD}_t  =  \argmin_{\lVert \theta \rVert_2 \leq S} \, \lVert g_{t-1}(\hat{\theta}^{\nD}_t) - g_{t-1}(\theta)
 \rVert_{\widetilde{W}_{t-1}^{-1}}\;.
\end{equation}

The second algorithm that we propose is $\D$: exponentially increasing weights are used
 to progressively forget the past. The theoretical aspects of this algorithm are detailed 
 in Section \ref{subsec:analysis_D}

\begin{algorithm}[h]
\caption{$\D$}
   \label{alg:D-GLM}
\begin{algorithmic}
   \STATE {\bfseries Input:} Probability $\delta$, dimension $d$, regularization $\lambda$, 
upper bound for actions $L$, upper bound for parameters $S$, discount factor $\gamma$.
\STATE {\bfseries Initialize:} $W_0 = \lambda/c_{\mu} I_d$, $\hat{\theta}^{\nD}_0 = 0_{\mathbb{R}^d}$.
   \FOR{$t=1$ {\bfseries to} $T$}
   \STATE Receive $\mathcal{A}_t$, compute $\hat{\theta}_t^{\nD}$ according to (\ref{eq_MLE_D})
   \STATE {\bfseries if} $\lVert \hat{\theta}_t^{\nD} \rVert_2 \leq S$ {\bfseries let} $\tilde{\theta}_t^{\nD} =  \hat{\theta}_t^{\nD}$
    {\bfseries else} compute
   $\tilde{\theta}_t^{\nD}$ with (\ref{eq:theta_tilde_D})
   \STATE {\bfseries Play} $A_t = \argmax_{a \in \mathcal{A}_t} \left( \mu(a^{\top} \tilde{\theta}^{\nD}_t) \hspace{-0.05cm}+ 
   \hspace{-0.05cm} \rho^{\nD}_t(\delta) \lVert a \rVert_{W_{t-1 }^{-1}} \right)$
  \STATE {\bfseries Receive} reward $X_t$
  \STATE {\bfseries Update:} $W_t \leftarrow A_t A_t^{\top} + \gamma W_{t-1} + \frac{\lambda}{c_{\mu}} (1- \gamma) I_d$
   \ENDFOR
\end{algorithmic}
\end{algorithm}

The parameter $\rho^{\nD}$ (line 8 above) will be defined below in Equation \eqref{eq_rho_D}.

\underline{Remark:} In the linear setting, the form of the upper
confidence bound is a direct consequence of
the high probability confidence ellipsoid that can be built around the estimate
of the unknown parameter \citep{abbasi2011improved}. There is no such confidence ellipsoid for
generalized linear bandits. Therefore, the upper confidence bound has a different form.
A possible approach that is chosen here is to consider
$\textnormal{UCB}_t(a) = \mathbb{E} _{\widetilde{\theta}_t} \left[ X_t
  | A_t = a \right] + \rho(t) \lVert a \rVert_{M_{t-1}^{-1}}$, where
$\mathbb{E} _{\widetilde{\theta}_t} [ X_t | A_t = a ]$ is equal to
$\mu(a^{\top} \widetilde{\theta}_t)$ under a GLM. For $\SW$,
$\rho = \rho^{\nSW}$ and $M_{t-1} = V_{t-1}$, as defined in Equation
\eqref{eq_Design_matrix_SW}. Similarly, for $\D$, $\rho = \rho^{\nD}$
and $M_{t-1}= W_{t-1}$ is defined in Equation
\eqref{eq_Design_matrix_D}.

\section{Concentration Bounds and Regret Analysis}

In this section we give concentration results for the two 
estimators that we propose.  Based on these concentration 
results, high probability upper-bounds for the dynamic
regret of both algorithms are given. We show that we obtain results comparable
to the ones in the linear setting. The main difference is that our
analysis is valid only for abruptly changing environments. Proposing
an algorithm that can be analyzed in both slowly drifting and abruptly
changing environments under a generalized linear bandit remains an
open question.

\subsection{Analysis of $\SW$}
\label{subsec:analysis_SW}

To obtain concentration inequalities, we need to restrict ourselves
 to segments of observations that are sufficiently far away from the 
 changepoints. More precisely, let
\begin{equation}
\label{D_tau_def}
\mathcal{T}(\tau) = \{t \leq T \,\,  \textnormal{s. t.} \, \, \forall \, t- \tau \leq s 
\leq t, \theta^{\star}_s = \theta^{\star}_t  \}\;.
 \end{equation}

$\mathcal{T}(\tau)$ contains all the time instants that are at least
$\tau$ steps away from the closest previous breakpoint. At time instants in 
$\mathcal{T}(\tau)$, there is no bias due to
non-stationarity of the environment as the sliding window of
length $\tau$ is fully included in a stationary segment.
 
\begin{proposition}
\label{prop_SW_GLM_concentration}
Let $0<\delta <1$ and  
$t \in \mathcal{T}(\tau)$. Let $\tilde{A}_t$ be any 
$\mathcal{A}_t$-valued random variable. Let

\begin{equation*}
\textnormal{c}_t^{\nSW}(\delta) = \frac{\X}{2} \sqrt{2 \log(T/\delta) + d \log\left( 1
 + \frac{c_{\mu} L^2 \min(t, \tau)}{d \lambda} \right) } 
\end{equation*}
\begin{equation}\hbox{and\quad }
\label{eq_rho_SW}
\rho_t^{\nSW}(\delta) = \frac{2 k_{\mu}}{c_{\mu}}  \bigg( 
\textnormal{c}^{\nSW}_t(\delta)
+ \sqrt{c_{\mu} \lambda} S \bigg)  \;.
\end{equation}

Then, simultaneously for all $t \in \mathcal{T}(\tau)$,
$$
\big|\mu(\tilde{A}_t^{\top} \theta^{\star}_t) - \mu(\tilde{A}_t^{\top} 
\widetilde{\theta}_t^{\nSW}) \big| \leq \rho_t^{\nSW}(\delta)\lVert \tilde{A}_t \rVert_{V_{t-1}^{-1}},
$$
holds with probability higher than $1- \delta$.

\end{proposition}

\begin{proofsketch}
  Only a proof sketch is given here: the complete proof is to be found in
  Appendix \ref{subsection:prop_SW_GLM_concentration}.  The
  big picture is to use the assumption on the inverse
  link function and on the MLE to relate the deviations of the regression
  estimate to those of the martingale
  $S_{t-1} =\sum_{s=\max(1, t-\tau)}^{t-1} A_s \eta_s$. For
  $t \in \mathcal{T}(\tau)$, this can be done by upper bounding
  $|\mu(\tilde{A}_t^{\top} \theta^{\star }_t) - \mu(\tilde{A}_t^{\top}
  \tilde{\theta}_t )|$ by the quantity
  $2k_{\mu}/{c_{\mu}} \lVert a \rVert_{V_{t-1}^{-1}} (\lVert S_{t-1
  }\rVert_{V_{t-1}^{-1}} + \lVert \lambda \theta^{\star}_t
  \rVert_{V_{t-1}^{-1}}) $.  Then, the concentration result is
  established by upper-bounding the self-normalized quantity $\lVert S_{t-1
  }\rVert_{V_{t-1}^{-1}}$.
\end{proofsketch}

The concentration result of Proposition \ref{prop_SW_GLM_concentration} is a prerequisite 
to give a high probability
upper-bound on the instantaneous regret
$\max_{a \in \mathcal{A}_t} \mu(a^{\top} \theta^{\star}_t)- \mu(A_{t}
^{\top} \theta^{\star}_t)$.

\begin{corollary} 
\label{prop_SW_GLM_anytime_upper}
Let $0<\delta<1$ and $A_{t, \star} =  \displaystyle{\argmax_{a \in \mathcal{A}_t}} \mu(a^{\top}
 \theta^{\star}_t)$. Then,
simultaneously for all $t \in \mathcal{T}(\tau)$
$$
\mu(A_{t,\star}^{\top} \theta^{\star}_t)- \mu(A_{t}^{\top} \theta^{\star}_t) \leq 2 \rho^{\nSW}_t(\delta) \lVert A_t \rVert_{V_{t-1}^{-1}}
$$
holds with probability at least $1-2\delta$.
\end{corollary}

The proof of this result is available in Appendix~\ref{subsec:prop_corollary_prop_SW_GLM_anytime}. Corollary
\ref{prop_SW_GLM_anytime_upper} allows us to give a high probability
upper bound on the instantaneous regret for all time instants far
enough from any breakpoints $ t \in \mathcal{T}(\tau)$. 

Based on those two concentration results, we can establish the following theorem for
the regret of $\SW$.

\begin{theorem}[Regret of $\SW$]
\label{theorem_regret_SW}
The regret of the $\SW$ policy is upper-bounded with probability  $\geq 1- 2\delta$ by
\begin{equation*}
\begin{split}
R_T &\leq 2\sqrt{2} \rho^{\nSW}_T(\delta)\sqrt{T} \sqrt{d\ceil{T/ \tau} \log \left( 1 + \frac{c_{\mu} L^2 \tau} {d \lambda} \right)} \\
& \quad  + \X \Gamma_T \tau \;,
\end{split}
\end{equation*}
where
$\rho^{\nSW}$ is defined in Equation \eqref{eq_rho_SW} and $\Gamma_T$ is the number of changes up to time $T$.
% \begin{comment}
% \begin{equation*}
% \begin{split}
% \rho(T) &= \frac{2 k_{\mu}}{c_{\mu}} \bigg( X_m \sqrt{2\log(2T/\delta)
% + d \log\left( 1 + \frac{c_{\mu} L^2 \tau}{d \lambda} \right) } \\
%  & \hspace{1.5cm}+ \sqrt{c_{\mu} \lambda} S
%  \bigg).
% \end{split}
% \end{equation*}
% \end{comment}
\end{theorem}

\begin{proof}
In the following proof let $\rho$ denote $\rho^{\nSW}$.
\begin{equation*}
\begin{split}
R_T &= \sum_{t=1}^T (\mu(A_{t,\star}^{\top} \theta^{\star}_t) - \mu(A_{t}^{\top} \theta^{\star}_t)) \\ 
& \leq \X \Gamma_T \tau + \hspace{-0.2cm}\sum_{t \in \mathcal{T}(\tau)}  \min\{\X ,\mu(A_{t,\star}^{\top} \theta^{\star}_t) - \mu(A_{t}^{\top} \theta^{\star}_t)\} 
\end{split}
\end{equation*}

where in the last inequality the instantaneous regret
$ \forall t \not \in \mathcal{T}(\tau)$ was upper-bounded by
$\X$. Using Corollary \ref{prop_SW_GLM_anytime_upper} with probability
$\geq 1- 2\delta$

\begin{equation*}
\begin{split}
R_T &\leq  \X \Gamma_T \tau +   \sum_{t \in \mathcal{T}(\tau)} \min\{\X, 2 \rho_t(\delta) \lVert A_t \rVert_{V_{t-1}^{-1}} \} \\
&\leq  \X \Gamma_T \tau +   2 \rho_T(\delta) \sum_{t \in \mathcal{T}(\tau)} \min\{1, \lVert A_t \rVert_{V_{t-1}^{-1}} \} \\
&\leq  \X \Gamma_T \tau +   2 \rho_T(\delta) \sum_{t=1}^T  \min\{1, \lVert A_t \rVert_{V_{t-1}^{-1}} \} \\
&\leq  \X \Gamma_T \tau +   2 \rho_T(\delta) \sqrt{T} \sqrt{\sum_{t=1}^T  \min\{1, \lVert A_t \rVert_{V_{t-1}^{-1}}^2\}} \;.
\end{split}
\end{equation*}
The second inequality holds thanks to $\X \leq 2 \rho(T)$ and the last
inequality is Cauchy–Schwarz. Proposition 9 in Appendix C
of \citet{russac2019weighted} yields
$$
\sum_{t=1}^T  \min\{1, \lVert A_t \rVert_{V_{t-1}^{-1}}^2\} \leq 2d \ceil{T/ \tau} \log \left( 1 + \frac{c_{\mu} L^2 \tau} {\lambda d} \right) ,
$$
which concludes the proof.
\end{proof}

In the following corollary, we denote $\tilde{O}$ the function growth when omitting the logarithmic terms.
\begin{corollary}
\label{corollary:asympt_regret_SW}
If $\Gamma_T$ is known, by choosing
$\tau \hspace{-0.05cm}=
\hspace{-0.05cm}\ceil{(\frac{dT}{\Gamma_T})^{2/3}}$, the regret of the
$\SW$ algorithm is asymptotically upper bounded with high probability
by a term $\tilde{O}(d^{2/3} \Gamma_T^{1/3} T^{2/3})$.

If $\Gamma_T$ is unknown, by choosing $\tau = \ceil{d^{2/3}T^{2/3}}$,
the regret of the $\SW$ algorithm is asymptotically upper bounded with
high probability by a term $\tilde{O}(d^{2/3} \Gamma_T T^{2/3})$.
\end{corollary}

This corollary is proved in Appendix \ref{subsec:corollary_regret_SW_asymptotic}.

\subsection{Analysis of $\D$}
\label{subsec:analysis_D}

The main difference when establishing concentration results in the
weighted setting is the need to control the bias term which was
avoided with the sliding window thanks to the condition
$t \in \mathcal{T}(\tau)$. In the weighted setting, with a discount factor
$\gamma$, we introduce
$ \mathcal{T}(\gamma)$ defined as
$$
\mathcal{T}(\gamma) = \{t \leq T \,\,  \textnormal{s. t.} \, \,
 \forall \, t- D(\gamma) < s \leq t, \theta^{\star}_s = \theta^{\star}_t  \}\;,
$$
where $D(\gamma)$ is an analysis parameter that will be specified later. The
main reason for introducing this parameter is to control the
bias. Basically, as in the linear setting, the bias for time instants far
enough from a breakpoint can be upper bounded more roughly than for the others.

\begin{proposition}
\label{prop_D_GLM_concentration}
Let $0< \delta < 1$ and. Let $\tilde{A}_t$ be any 
$\mathcal{A}_t$-valued random variable. Let

\begin{equation*}
\textnormal{c}_t^{\nD}(\delta) = \frac{\X}{2} \sqrt{2 \log(1/\delta) + d \log\left( 1 +
 \frac{c_{\mu} L^2 (1- \gamma^{2t})}{d \lambda (1 - \gamma^2)} \right) }\;,
\end{equation*}
\begin{equation}
\label{eq_rho_D}
\rho_t^{\nD}(\delta) = \frac{2 k_{\mu}}{c_{\mu}}\bigg(  \textnormal{c}_t^{\nD}(\delta)  +  
\sqrt{c_{\mu} \lambda} S +  2L^2 S k_{\mu} \sqrt{\frac{c_{\mu}}{\lambda}}  \frac{\gamma^{D(\gamma)}}{1-\gamma}  \bigg) \;.
\end{equation}
Then simultaneously for all $t \in \mathcal{T}(\gamma)$ 
$$
|\mu(\tilde{A}_t^{\top} \theta^{\star}_t) - \mu(\tilde{A}_t^{\top} 
\tilde{\theta}_t^{\nD}) | \leq \rho_t^{\nD}(\delta) \lVert \tilde{A}_t \rVert_{W_{t-1}^{-1}}\;,
$$
holds with a probability higher than $1-\delta$.
\end{proposition}

\begin{proofsketch}

  As with the sliding window, we would like to use the concentration
  results established in the linear setting and extend the analysis to
  GLMs. The first step consists in upper bounding
  $|\mu(\tilde{A}_t^{\top} \theta^{\star}_t) - \mu(\tilde{A}_t^{\top}
  \tilde{\theta}^{\nD}_t) | $ with assumption
  \ref{assumption_c_mu}. The upper-bound is a sum of two main
  terms. The first one is related to the weighted martingale
  $S_{t-1} = \sum_{s=1}^{t-1} \gamma^{-s} A_s \eta_s$. The
  self-normalized quantity
  $\lVert S_{t-1} \rVert_{\gamma^{2(t-1)} \widetilde{W}_{t-1}^{-1}}$
  can be upper-bounded with high probability and we use Corollary
  \ref{corollary:S_t} to do so. The next step consists in controlling
  the bias
  $ \lVert \sum_{s=1}^{t-1-D(\gamma)} \gamma^{t-s} (\mu(A_s^{\top}
  \theta^{\star}_t) - \mu(A_s^{\top} \theta^{\star}_s))A_s \rVert_{
    \widetilde{W}^{-1}_{t-1}}$. The assumption
  $t \in \mathcal{T}(\gamma)$ is required at this step to have a
  proper control on this term. By combining the Lipschitz assumption
  (Assumption \ref{assumption_c_mu})
  on the inverse link function and a triangle inequality, the bias
  term can be upper-bounded by
  $2L^2 S k_{\mu} \sqrt{c_{\mu}/\lambda}
  \gamma^{D(\gamma)}/(1-\gamma)$. A detailed proof is available in
  Appendix \ref{subsection:prop_D_GLM_concentration}
\end{proofsketch}

\underline{Remark:} In the linear setting, the bias can be controlled
independently from the stochastic term. For example,
\citet{russac2019weighted} consider a confidence ellipsoid centered
around $\bar{\theta}_t$ (Proposition 3 of \citet{russac2019weighted})
to separate the two terms.
With the particular geometry of the GLMs this is not achievable 
with the estimator we considered and
the bias appears explicitly in the confidence bound as an additive
term.

Proposition \ref{prop_D_GLM_concentration} can now be used
to obtain a high
probability upper bound for the instantaneous regret for all time
instants $t \in \mathcal{T}(\gamma)$. We have the following corollary.

\begin{corollary}
\label{prop_D_GLM_anytime_upper}
Let $0<\delta<1$, and $A_{t, \star} =  \argmax_{a \in \mathcal{A}_t} \mu(a^{\top} \theta^{\star}_t)$. Then,
simultaneously for all $t \in \mathcal{T}(\gamma)$
$$
\mu(A_{t,\star}^{\top} \theta^{\star}_t)- \mu(A_{t}^{\top} \theta^{\star}_t) \leq 2 \rho_t^{\nD}(\delta) \lVert A_t \rVert_{W_{t-1}^{-1}}
$$
holds with probability at least $1-2\delta$.
\end{corollary}

The proof of this corollary essentially follows the ideas of the proof
of Corollary \ref{prop_SW_GLM_anytime_upper}. The main difference is
the term $\log(1/\delta)$ in the high probability upper-bound (in
$c_t^{\nD}(\delta)$) instead of $\log(T/\delta)$ (in
$c_t^{\nSW}(\delta)$).  This is because in the weighted setting an
anytime deviation bound can be obtained (Corollary
\ref{corollary_anytime_deviation} in Appendix).  On the contrary, with the sliding
window, we cannot avoid the union bound argument to obtain the
concentration result valid for all $t \in \mathcal{T}(\tau)$ which
gives the extra $T$ term.

The reader familiar with the analysis in the weighted linear setting may
be surprised by the presence of the term $\lVert a\rVert_{W_{t-1}^{-1}}$ in the
exploration bonus for $\D$. In fact, one of the conclusion of
\citet{russac2019weighted} was to prove that the exploration term in
the upper confidence bound
must contain the $W_{t-1}^{-1} \widetilde{W}_{t-1} W_{t-1}^{-1}$ norm of $A_t$ 
 (with $c_{\mu} = 1$ in
the linear setting). However,
knowing that $0<\gamma<1$, we have $\gamma^{2(t-s)} \leq \gamma^{t-s}$
for $s \leq t$, implying that $\widetilde{W}_{t-1} \leq
W_{t-1}$. Consequently,
$\lVert a \rVert_{W_{t-1}^{-1} \widetilde{W}_{t-1} W_{t-1}^{-1}} \leq
\lVert a \rVert_{W_{t-1}^{-1}} $. The take home message is that it is
possible to obtain a tighter bound in the linear case with a control
in the $W_{t-1} \widetilde{W}_{t-1}^{-1} W_{t-1}$ norm for the
confidence ellipsoid (Theorem 1 of \citet{russac2019weighted}), while the
exploration term features the $W_{t-1}^{-1}$ norm in the GLM.

%\newpage
\begin{theorem}[Regret of $\D$]
\label{theorem_regret_D}
The regret of the $\D$ policy is upper-bounded with probability  $\geq 1- 2\delta$ by
\begin{equation*}
\begin{split}
R_T \leq  & 2\rho^{\nD}_T(\delta)\sqrt{2dT} \sqrt{T \log \left(\frac{1}{\gamma}\right)\hspace{-0.05cm}+\hspace{-0.05cm}\log\left(1 \hspace{-0.05cm}+\hspace{-0.05cm} 
\frac{c_{\mu} L^2}{d\lambda(1-\gamma)} \right)} \\ 
& + \X \Gamma_T D(\gamma) \;,
\end{split}
\end{equation*}
where
$\rho^{\nD}$ is defined in Equation \eqref{eq_rho_D} and $\Gamma_T$ is the number of changes up to time $T$.
\end{theorem}

The proof essentially follows the arguments presented in Theorem
\ref{theorem_regret_SW} and is reported in Appendix
\ref{subsection:regret_D}

\begin{corollary}
 \label{corollary:asympt_regret_D}
 By taking $D(\gamma) = \frac{\log(1/(1-\gamma))}{1- \gamma}$,
 \begin{enumerate}
 \item If $\Gamma_T$ is known, by choosing
   $\gamma = 1- (\frac{\Gamma_T}{dT})^{2/3}$, the regret of the $\D$
   algorithm is asymptotically upper bounded with high probability by
   a term $\tilde{O}(d^{2/3} \Gamma_T^{1/3} T^{2/3})$.
\item If $\Gamma_T$ is unknown, by choosing
   $\gamma= 1- (\frac{1}{dT})^{2/3}$, the regret of the $\D$ algorithm
   is asymptotically upper bounded with high probability by a term
   $\tilde{O}(d^{2/3} \Gamma_T T^{2/3})$.
 \end{enumerate}
\end{corollary}
This corollary is proved in Appendix \ref{subsec:asympt_regret_D}.

\section{Experiments}

In this section, we evaluate the empirical performance of the two proposed algorithms. In a first part, we reproduce the simulation proposed 
in an abruptly changing environment in \citet{russac2019weighted}. It consists in a two-dimensional problem 
with 3 different breakpoints. 
The theoretical aspects developed in the previous sections suggest that $\SW$ and $\D$ should have
better performance than generalized linear bandit algorithms that do not 
take into account the non-stationarity. In a second part, we use a real world dataset
 to test the performances of the algorithms on a 9-dimensional problem where non-stationarity 
 is artificially created.

\subsection{Simulated environment}

In this simulated environment, we compare different generalized linear bandits algorithms and 
linear bandits algorithms when
 the inverse link function is the sigmoid $\mu(x) = 1/(1+ \exp(-x))$: the $\SW$
 algorithm using a sliding window, the $\D$ algorithm based on the use
  of exponentially increasing weights and the stationary
  algorithm, where the maximum likelihood estimator is solution of Equation 
  \eqref{eq_log_likelihood}.
  Additionally 
  to those three algorithms, we add their
  linear counterpart, $\Lin$ as in \cite{abbasi2011improved}, $\SWLin$ as in
   \citep{cheung2019learning} and $\DLin$ as presented in \citep{russac2019weighted}.
  Those three algorithms do not assume that the rewards are generated 
  by a logistic function and use a misspecified linear model; we expect them to have higher regrets.
  
    \begin{figure}[h]
        \centering
        \includegraphics[width=0.42\textwidth]{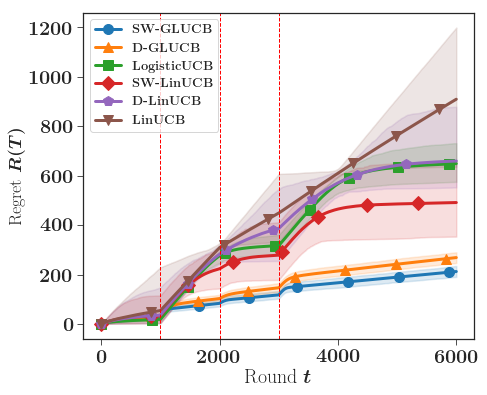}
        \caption{Regret of the different algorithms in a 2D abruptly changing environment
        and the $5\%$ quantiles averaged on 500 independent runs}
        \label{fig:regret_2d}
\end{figure}

  In this experiment the number of rounds is set to $T=6000$. $\theta^{\star}_t$ 
  the parameter in the logistic function is evolving over time: before $t=1000$,
  $\theta^{\star}_t = (1,0)$; for $1001 \leq t \leq 2000$,    $\theta^{\star}_t = (-1,0)$; 
  for $2001 \leq t \leq 3000$,    $\theta^{\star}_t = (0,1)$ 
  and for $ t > 3000$,    $\theta^{\star}_t = (0,-1)$. The position
   of $\theta^{\star}$ at the different periods are represented by the light blue triangles in the 
  scatter plot in Figure
  \ref{fig:scatter_2d}. 
    The locations of the changepoints are also represented on Figure \ref{fig:regret_2d}
     by the red dashed vertical lines. 
     In this problem, $\theta^{\star}$ is widely spread over
      the 2 dimensional unit ball. At each round $K=6$ actions randomly
       generated 
  in the unit ball are presented to the different algorithms. The instantaneous
   regret in round $t$ is defined as $r_t = \max_{a \in \{A_{t,1},... A_{t,6} \} } 
  \mu(a^{\top} \theta^{\star}_t) -  \mu(A_t^{\top} \theta^{\star}_t) $, where
   $A_t$ is the action chosen by the algorithm.
  In Figure \ref{fig:regret_2d} the cumulative dynamic regret of
   the different algorithms averaged on 500 independent runs is represented.  The shaded region 
   correspond to the $5\%$ and the $95\%$ quantiles for the cumulative regrets of the different algorithms. 
   We can see that the variation of the performance is much larger for linear bandits algorithms than
   for the generalized linear bandits algorithms, a potential reason for this is that the confidence ellipsoid
   for the linear algorithms do not hold if the linear assumption of the rewards is not satisfied.

 \begin{figure}[hbt]
        \centering
        \includegraphics[width=0.41\textwidth]{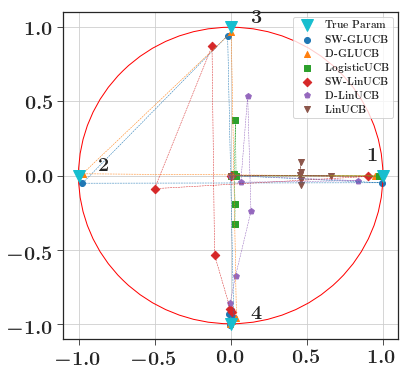}
        \caption{Estimated parameter ($\hat{\theta}_t$) every 1000 steps for 
        the different algorithms in a 2D abruptly changing environment averaged on 500 independent runs.}
        \label{fig:scatter_2d}
\end{figure}

 After the different changepoints a 3000 rounds stationary period is
 added to check if the estimators of the different algorithms converge to the true parameter. 
In Figure \ref{fig:scatter_2d}, the estimator $\hat{\theta}_t$ is plotted 
every 1000 rounds for the different algorithms. We expect well-performing algorithms
to approach the ground truth $\theta^{\star}$.
The evolution of $\theta^{\star}_t$ requires the different algorithms to adapt to the changes.
$\Log$ and $\Lin$ fail in doing so. The failure is even worse for $\Lin$ because 
the algorithm does not leverage the logistic function information and does not converge, even after the 
stationary period corresponding to the second half of the experiment. On the scatter plot,
 the estimator for $\Lin$ never approaches the ground truth which 
explains the important regret. The $\Log$ estimator catches the ground truth in the first stationary 
period but is not able to adapt to the changes in $\theta^{\star}$ and fails in estimating the evolving parameter. 
If the final stationary period is longer, it will eventually build a better estimator and converge.

The best performing policies are $\SW$ and $\D$. The estimators built in
 those algorithms track the evolving parameter accurately as can be seen on the scatter plot
 on Figure \ref{fig:scatter_2d}.  $\SWLin$
  performs surprisingly well. By progressively forgetting the past, the algorithm builds quite precise estimate 
  of $\theta^{\star}$. Of course, the algorithm is not as precise as $\SW$ because it doesn't rely on the additional logistic 
  assumption on the rewards, which implies a slower convergence to the true unknown parameter.

\subsection{Simulation with a real-world dataset}

In this section, we illustrate the performance of the generalized linear bandits algorithms with a real dataset. In contrast
with the previous simulated environment, the rewards here are not generated by a logistic function but 
are the target variable of the dataset.
We use the Pima Indian Diabetes Database \footnote{The dataset can be
 downloaded \href{https://www.kaggle.com/uciml/pima-indians-diabetes-database}{here}.} 
where the  aim is to predict if a patient has diabetes or not. The 
predictions are based on 8 variables characterizing the different 
patients: number of pregnancies,
the glucose level, the blood pressure, the thickness of the skin, insulin, 
the body mass index, the diabetes pedigree function and the age.

All the variables are numerical and the processing step consists
 in centering and standardizing the different variables.
The outcome variable is binary and has the value 1 if the
 patient has diabetes. We run a 2000 steps experiment
designed as follows: at each round, a 
patient without diabetes and a patient with diabetes are 
randomly selected and  proposed to the different algorithms. The reward is $+1$ 
if the patient with diabetes was selected by the algorithm.
We artificially create non-stationarity by inverting the population of diabetic and non-diabetic patients
at time $t= 1000$.
This change corresponds to a large perturbation but the algorithms
that progressively forget the past should be able to adapt to the
change and progressively recover a classification performance
comparable to the level attained in the first segment.

\begin{figure}[hbt]
        \centering
        \includegraphics[width=0.45\textwidth]{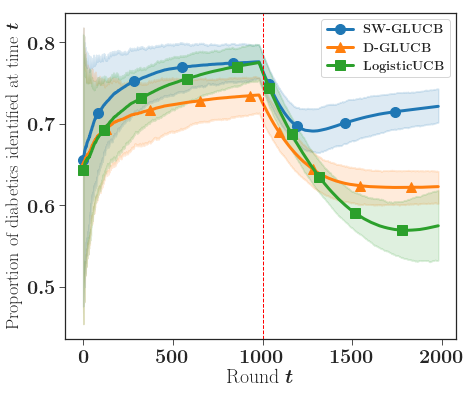}
        \caption{Proportion of diabetic patients detected at time $t$ in
         an artificially created non-stationary environment averaged on 500 independent runs.}
        \label{fig:PIMA}
\end{figure}

We report in Figure \ref{fig:PIMA} the proportion of
 diabetic patients detected averaged on 500 independent runs. Here, contrarily to the simulated 
 environment, the rewards are not generated 
 with a logistic function but the taken from the original dataset. Hence, we cannot directly evaluate the regret and the learning is more complex 
 because the model misspecified. Nevertheless, even in this
  setting $\SW$ and $\D$ are able to learn continuously.
After the changepoint, the diabetic patients are harder to detect
 and the averaged cumulative sum decreases for all the algorithms. The recovery is much faster
 for $\SW$ and $\D$ than for the stationary logistic bandit model.
Although very simplistic, this experiment suggests that the proposed
algorithms are robust enough to be successfully used for online
bandit learning in realistic non-stationary binary regression tasks.
 
%\clearpage
\bibliography{preprint}
\bibliographystyle{abbrvnat}

\clearpage

\onecolumn
\appendix

\begin{center}
\large
Supplementary for \\
\Large
Algorithms for Non-Stationary Generalized Linear Bandits
\vspace{10mm}
\end{center}

\section{Proof for the sliding window GLM}

\subsection{Proof of Proposition \ref{prop_SW_GLM_concentration}}
\label{subsection:prop_SW_GLM_concentration}

  \begin{customprop}{1}
Let $0<\delta <1$ and 
$t \in \mathcal{T}(\tau)$. Let $\tilde{A}_t$ be any 
$\mathcal{A}_t$-valued random variable. Let

\begin{equation*}
\textnormal{c}_t^{\nSW}(\delta) = \frac{\X}{2} \sqrt{2 \log(T/\delta) + d \log\left( 1 + \frac{c_{\mu} L^2 \min(t, \tau)}{d \lambda} \right) } 
\end{equation*}
\begin{equation*}\hbox{and \quad}
\rho_t^{\nSW}(\delta) = \frac{2 k_{\mu}}{c_{\mu}}  \bigg( \textnormal{c}^{\nSW}_t(\delta)
+ \sqrt{c_{\mu} \lambda} S \bigg)  \;.
\end{equation*}

Then, simultaneously for all $t \in \mathcal{T}(\tau)$,
$$
\big|\mu(\tilde{A}_t^{\top} \theta^{\star}_t) - \mu(\tilde{A}_t^{\top} 
\widetilde{\theta}_t^{\nSW}) \big| \leq \rho_t^{\nSW}(\delta)\lVert \tilde{A}_t \rVert_{V_{t-1}^{-1}}
$$
holds with probability higher than $1- \delta$.
\end{customprop}
  
\begin{proof}
We define $g_{t-1}: \mathbb{R}^d \mapsto \mathbb{R}^d$ by $g_{t-1}(\theta) = 
\sum_{s= \max(1, t- \tau)}^{t-1} \mu(A_s^{\top} \theta) A_s + \lambda \theta$. Let $J_{t-1}$ denotes the Jacobian matrix of $g_{t-1}$.  We have
$J_{t-1}(\theta) = \sum_{s=\max(1, t- \tau)}^{t-1} \dot{\mu}(A_s^{\top} \theta) A_s A_s^{\top} + \lambda I_d$.

Thanks to the definition of the estimator $\hat{\theta}^{\nSW}_t$ defined in Equation \eqref{eq_MLE_SW}, we have $g_{t-1}(\hat{\theta}_t^{\nSW}) =
 \sum_{s= \max(1, t-\tau)}^{t-1} A_s X_s$. We also introduce the martingale $S_{t-1} = \sum_{s= \max(1, t- \tau)}^{t-1} A_s \eta_s$. In the following proof,
 we use $\tilde{\theta}_t$ instead of $\tilde{\theta}_t^{\nSW}$.
 
 We define the $G_{t-1}(\theta^{\star}_t,\tilde{\theta}_t)$ matrix as follows,
 $$
 G_{t-1}(\theta^{\star}_t,\tilde{\theta}_t) = \int_{0}^1 J_{t-1}( u \theta^{\star}_t + (1-u) \tilde{\theta}_t) \, du\;.
 $$
 The Fundamental Theorem of Calculus gives
 \begin{equation}
\label{eq_g_t_G_t} 
 g_{t-1}(\theta^{\star}_t) -  g_{t-1}(\tilde{\theta}_t)  = G_{t-1}(\theta^{\star}_t,\tilde{\theta}_t)(\theta^{\star}_t- \tilde{\theta}_t)\;.
  \end{equation}
Knowing that both $\theta^{\star}_t$ and $\tilde{\theta}_t$ have an L2-norm smaller than
 $S$, $\forall u \in [0,1],  \lVert u \theta^{\star}_t + (1-u) \tilde{\theta}_t) \rVert_2 \leq S$.
This implies in particular that
\begin{equation}
\label{eq_G_t}
G_{t-1}(\theta^{\star}_t,\tilde{\theta}_t) \geq c_{\mu} \left( \sum_{s= \max(1,t-\tau)}^{t-1} A_s A_s^{\top} + \frac{\lambda} {c_{\mu}} I_d  \right) = c_{\mu} V_{t-1}\;,
\end{equation}
which in turn ensures $G_{t-1}(\theta^{\star}_t,\tilde{\theta}_t)$ is invertible.

Let $\tilde{A}_t$ be any $\mathcal{A}_t$ valued random variable and $t$ be a fixed time instant,
  \begin{align*}
 |\mu(\tilde{A}_t^{\top}
 \theta^{\star}_t) - \mu(\tilde{A}_t^{\top} 
\tilde{\theta}_t) | &\leq k_{\mu}|\tilde{A}_t^{\top}(\theta^{\star}_t-
 \tilde{\theta}_t) | \quad \textnormal{(Assumption \ref{assumption_c_mu})}\\
 &= k_{\mu} |\tilde{A}_t^{\top}  G_{t-1}^{-1}(\theta^{\star}_t,\tilde{\theta}_t)
 (g_{t-1}(\theta^{\star}_t)-g_{t-1}(\tilde{\theta}_t))| \quad \textnormal{(Equation (\eqref{eq_g_t_G_t}))} \\
 &\leq k_{\mu} \lVert \tilde{A}_t \rVert_{G_{t-1}^{-1}(\theta^{\star}_t,\tilde{\theta}_t)} \lVert g_{t-1}(\theta^{\star}_t)-g_{t-1
 }(\tilde{\theta}_t)
 \rVert_{G_{t-1}^{-1}(\theta^{\star}_t,\tilde{\theta}_t)} \quad \textnormal{(C-S)}\\
  &\leq \frac{k_{\mu}}{c_{\mu}}  \lVert \tilde{A}_t \rVert_{V_{t-1}^{-1}} \lVert g_{t-1}(\theta^{\star}_t)-g_{t-1}(\tilde{\theta}_t)\rVert_{V_{t-1}
  ^{-1}} \quad \textnormal{(Equation (\eqref{eq_G_t}))}  \\
&  \leq \frac{2k_{\mu}}{c_{\mu}}  \lVert \tilde{A}_t \rVert_{V_{t-1}^{-1}} \lVert g_{t-1}(\theta^{\star}_t)-g_{t-1}(\hat{\theta}_t^{\nSW}
)\rVert_{V_{t-1}^{-1}} \quad \textnormal{(Definition of $\tilde{\theta}_t$)} \\
  &\leq \frac{2k_{\mu}}{c_{\mu}}  \lVert \tilde{A}_t \rVert_{V_{t-1}^{-1}} 
  \lVert \sum_{s= \max(1, t-\tau)}^{t-1}  \mu(A_s^{\top} \theta^{\star}_t) A_s
  + \lambda \theta^{\star}_t - \sum_{s= \max(1, t-\tau)}^{t-1} A_s X_s  
  \rVert_{V_{t-1}^{-1}}  \\
 &\leq \frac{2k_{\mu}}{c_{\mu}}  \lVert \tilde{A}_t \rVert_{V_{t-1}^{-1}} 
  \lVert \sum_{s= \max(1, t-\tau)}^{t-1}  (\mu(A_s^{\top} \theta^{\star}_t) - \mu(A_s^{\top} \theta^{\star}_s) )A_s
  -\sum_{s= \max(1, t-\tau)}^{t-1} A_s \eta_s  + \lambda \theta^{\star}_t
  \rVert_{V_{t-1}^{-1}}  \\
 &\leq \frac{2k_{\mu}}{c_{\mu}}  \lVert \tilde{A}_t \rVert_{V_{t-1}^{-1}} 
  \lVert 
  -S_{t-1}  + \lambda \theta^{\star}_t
  \rVert_{V_{t-1}^{-1}} \quad \textnormal{( $t \in \mathcal{T}(\tau)$ )} \\
 &\leq \frac{2k_{\mu}}{c_{\mu}}  \lVert \tilde{A}_t \rVert_{V_{t-1}^{-1}} 
  \left(\lVert 
  S_{t-1} \rVert_{V_{t-1}^{-1}}  + \lVert \lambda \theta^{\star}_t
  \rVert_{V_{t-1}^{-1}} \right) \quad \textnormal{(Triangle inequality)}\\
   &\leq \frac{2k_{\mu}}{c_{\mu}}  \lVert \tilde{A}_t \rVert_{V_{t-1}^{-1}} 
  \left(\lVert 
  S_{t-1} \rVert_{V_{t-1}^{-1}}  + \sqrt{\lambda c_{\mu}} \lVert \theta^{\star}_t
  \rVert_{2} \right) \quad (V_{t-1} \geq \frac{\lambda}{c_{\mu}} I_d) \\
  &\leq \frac{2k_{\mu}}{c_{\mu}}  \lVert \tilde{A}_t \rVert_{V_{t-1}^{-1}} 
  \left( \frac{\X}{2} \sqrt{2\log(1/\delta) + d \log\left( 1 + \frac{c_{\mu} L^2 \min(t, \tau)}{d\lambda}\right)}+ \sqrt{\lambda c_{\mu}} S
   \right) \quad (\textnormal{with h.p.}) \;.
  \end{align*}
   
In the last inequality we have used the concentration result established
 in the Proposition 5 of \citet{russac2019weighted} for the self-normalized quantity $\lVert 
  S_{t-1} \rVert_{V_{t-1}^{-1}}$, and
 the assumption $\forall t,  \lVert \theta^{\star}_t \rVert_2 \leq S$. 
 To obtain the concentration result for all $t \in \mathcal{T}(\tau)$ we use a union bound.
 The final statement holds with probability $\geq 1-\delta$.
\end{proof}

\subsection{Proof of Corollary \ref{prop_SW_GLM_anytime_upper}}
\label{subsec:prop_corollary_prop_SW_GLM_anytime}

\begin{customcor}{1} 

Let $0<\delta<1$, and $A_{t, \star} =  \argmax_{a \in \mathcal{A}_t} \mu(a^{\top}
 \theta^{\star}_t)$. Then,
simultaneously for all $t \in \mathcal{T}(\tau)$
$$
\mu(A_{t,\star}^{\top} \theta^{\star}_t)- \mu(A_{t}^{\top} \theta^{\star}_t) \leq 2 \rho^{\nSW}_t(\delta) \lVert A_t \rVert_{V_{t-1}^{-1}}
$$
holds with probability at least $1-2\delta$.
\end{customcor}

\begin{proof} In the following proof,
 we abbreviate $\tilde{\theta}_t^{\nSW}$ to $\tilde{\theta}_t$.
\begin{align*}
\mu(A_{t,\star}^{\top} \theta^{\star}_t)- \mu(A_{t}^{\top} \theta^{\star}_t) = \underbrace{\mu(A_{t,\star}^{\top} \theta^{\star}_t)- \mu(A_{t,\star}^{\top}\tilde{\theta}_t)}_{A1} 
+\underbrace{\mu(A_{t,\star}^{\top}\tilde{\theta}_t)-  \mu(A_{t}^{\top} \tilde{\theta}_t)}_{A2}
+  \underbrace{\mu(A_{t}^{\top} \tilde{\theta}_t)-\mu(A_{t}^{\top} \theta^{\star}_t)}_{A3}
\end{align*}

Thanks to Proposition \ref{prop_SW_GLM_concentration}, we can give an upper bound for the term $A1$ and for the term $A3$.
Upper bounding $A2$ with high probability requires extra-work.

With a union bound, we can simultaneously upper bound $A1$ and $A3$ for all $t \in \mathcal{T}(\tau)$ and the following holds

\begin{equation}
\label{eq:simult_upper_bound}
\mathbb{P}\left(\forall t \in \mathcal{T}(\tau), \mu(A_{t,\star}^{\top} \theta^{\star}_t)- \mu(A_{t,\star}^{\top}\tilde{\theta}_t)
 \leq \rho^{\nSW}_t(\delta) \lVert A_{t,\star} \rVert_{V_{t-1}^{-1}} \,\cap 
\,  \mu(A_{t}^{\top} \tilde{\theta}_t)-\mu(A_{t}^{\top} \theta^{\star}_t) 
\leq \rho^{\nSW}_t(\delta)\lVert A_{t} \rVert_{V_{t-1}^{-1}} \right) \geq 1- 2\delta
\end{equation}

Let $E$ denote this event. The upper confidence at time $t$ for an action $a$ is defined by,
$$
\textnormal{UCB}_t(a) = \mu(a^{\top} \tilde{\theta}_t) + \rho^{\nSW}_t(\delta) \lVert a \rVert_{V_{t-1}^{-1}},
$$

The action chosen at time $t$,  $A_t$ is the action maximizing $\textnormal{UCB}_t(a)$ for $a \in \mathcal{A}_t$.

\begin{align*}
A2 &= \mu(A_{t,\star}^{\top}\tilde{\theta}_t)-  \mu(A_{t}^{\top} \tilde{\theta}_t) \\
&= \mu(A_{t,\star}^{\top}\tilde{\theta}_t) + \rho^{\nSW}_t(\delta) \lVert A_{t,\star} \rVert_{V_{t-1}^{-1}} - 
\rho^{\nSW}_t(\delta) \lVert  A_{t, \star} \rVert_{V_{t-1}^{-1}}
 -  \mu(A_{t}^{\top} \tilde{\theta}_t)  \\
&\leq \mu(A_{t}^{\top}\tilde{\theta}_t) + \rho^{\nSW}_t(\delta) \lVert A_{t} \rVert_{V_{t-1}^{-1}}
- \rho^{\nSW}_t(\delta) \lVert  A_{t, \star} \rVert_{V_{t-1}^{-1}}
 -  \mu(A_{t}^{\top} \tilde{\theta}_t)
 \quad \textnormal{(Definition of $A_t$)}\\
&\leq \rho^{\nSW}_t(\delta) \lVert A_{t} \rVert_{V_{t-1}^{-1}}
- \rho^{\nSW}_t(\delta) \lVert  A_{t, \star} \rVert_{V_{t-1}^{-1}}.
\end{align*}

Under the event $E$, that occurs with a probability higher than $1-\delta$,
\begin{align*}
\mu(A_{t,\star}^{\top} \theta^{\star}_t)- \mu(A_{t}^{\top} \theta^{\star}_t)  &\leq \underbrace{\rho^{\nSW}_t(\delta) 
\lVert A_{t,\star} \rVert_{V_{t-1}^{-1}}}_{\textnormal{coming from} A1} + 
\underbrace{\rho^{\nSW}_t(\delta) \lVert A_{t} \rVert_{V_{t-1}^{-1}}
- \rho^{\nSW}_t(\delta) \lVert  A_{t, \star} \rVert_{V_{t-1}^{-1}}}_{\textnormal{coming from} A2} 
+ \underbrace{\rho^{\nSW}_t(\delta) \lVert A_{t} \rVert_{V_{t-1}^{-1}}}_{\textnormal{coming from} A3} \\
&\leq 2 \rho^{\nSW}_t(\delta) \lVert A_t \rVert_{V_{t-1}^{-1}}.
\end{align*}
\end{proof}

\subsection{Proof of Corollary \ref{corollary:asympt_regret_SW}}
\label{subsec:corollary_regret_SW_asymptotic}
\begin{customcor}{2}
If $\Gamma_T$ is known, by choosing $\tau \hspace{-0.05cm}= \hspace{-0.05cm}\ceil{(\frac{dT}{\Gamma_T})^{2/3}}$, the regret of the $\SW$ algorithm is asymptotically upper bounded
with high probability by a term $\tilde{O}(d^{2/3} \Gamma_T^{1/3} T^{2/3})$.

If $\Gamma_T$ is unknown, by choosing $\tau = \ceil{d^{2/3}T^{2/3}}$, the regret of the $\SW$ algorithm is asymptotically upper bounded
with high probability by a term $\tilde{O}(d^{2/3} \Gamma_T T^{2/3})$.
\end{customcor}

 \begin{proof}
 With this particular choice of $\tau$ we have:
 $$ \tau \Gamma_T \sim d^{2/3} T^{2/3} \Gamma_T^{1/3} $$
$$ \rho^{\nSW}_T(\delta) \sim \sqrt{d \log(T)} $$
$$\sqrt{T} \sqrt{\ceil{T/\tau}} \sim d^{-1/3}T^{1-1/3} \Gamma_T^{1/3}$$

 Therefore the behavior of $ \rho^{\nSW}_T(\delta) \sqrt{dT} 
 \sqrt{\ceil{T/ \tau}} \sqrt{ \log\left(  1+ \frac{\tau L^2}{\lambda d} \right)}$  is similar 
 to $d^{2/3} \Gamma_T^{1/3} T^{2/3}\sqrt{\log(T)}\sqrt{\log(T/\Gamma_T)}$.

 By neglecting the logarithmic term, we have with high probability, 
$$ 
R_T = \widetilde{O}_{T \to \infty}(d^{2/3} \Gamma_T^{1/3}T^{2/3})\;.
$$
 \end{proof}

%%%%%%%%%%%%%%%%%%%%%%%%%%%%%
%%%%%%%%%%%%%%%%%%%%%%%%%%
%%%%%%%%%%%%%%%%%%%%
%%%%%%%%%%%%%%%
%%%%%%%%%%
%%%%%%
%%%
%

\section{Proof for the discounted GLM}

\subsection{Self-normalized concentration result}
\begin{corollary}[Corollary 3 of \citet{russac2019weighted}]
\label{corollary_anytime_deviation}
$\forall \delta > 0$, with $S_t = \sum_{s=1}^t \gamma^{-s} A_s \eta_s$, $\widetilde{V}_{t} = \sum_{s=1}^{t} \gamma^{-2s} A_s A_s^{\top} + \frac{\lambda \gamma^{-2t}}{c_{\mu}} I_d$ and  when $(\eta_s)_{s\geq 1}$ are $\sigma$-subgaussian conditionally on the past, we have
\label{corollary:S_t}
\begin{align*}
\mathbb{P}\left(\exists t \geq 0,\lVert S_t \rVert_{\widetilde{V}_t^{-1}}\geq \sigma\sqrt{ 2\log\left(\frac{1}{\delta}\right) +d \log\left(1 +
\frac{ c_{\mu} L^2 (1-\gamma^{2t})}{d \lambda  (1-\gamma^2)}\right)}   \right) \leq \delta\;.
\end{align*}
\end{corollary}

\begin{proof}
The proof is exactly the same than the one proposed in \citet{russac2019weighted}, except that $\tilde{\lambda } = \lambda / c_{\mu} $ is used rather than $\lambda$, which explains the slight difference in the formula proposed in Corollary \ref{corollary:S_t} compared to the original lemma.
\end{proof}

\subsection{Proof of Proposition \ref{prop_D_GLM_concentration}}
\label{subsection:prop_D_GLM_concentration}

\begin{customprop}{2}
Let $0< \delta < 1$ and. Let $\tilde{A}_t$ be any 
$\mathcal{A}_t$-valued random variable. Let

\begin{equation*}
\textnormal{c}_t^{\nD}(\delta) = \frac{\X}{2} \sqrt{2 \log(1/\delta) + d \log\left( 1 + \frac{c_{\mu} L^2 (1- \gamma^{2t})}{d \lambda (1 - \gamma^2)} \right) }
\end{equation*}
\begin{equation*}\hbox{and \quad}
\rho_t^{\nD}(\delta) = \frac{2 k_{\mu}}{c_{\mu}}\bigg(  \textnormal{c}_t^{\nD}(\delta) + \sqrt{c_{\mu} \lambda} S + 2L^2 S k_{\mu} \sqrt{\frac{c_{\mu}}{\lambda}}  \frac{\gamma^{D(\gamma)}}{1-\gamma}  \bigg)\;.
\end{equation*}
Then simultaneously for all $t \in \mathcal{T}(\gamma)$ 
$$
|\mu(\tilde{A}_t^{\top} \theta^{\star}_t) - \mu(\tilde{A}_t^{\top} 
\tilde{\theta}_t^{\nD}) | \leq \rho_t^{\nD}(\delta) \lVert \tilde{A}_t \rVert_{W_{t-1}^{-1}}\;,
$$
holds with a probability higher than $1-\delta$.
\end{customprop}

\begin{proof}

During the proof, when no confusion is possible, we will forget the upper-script for the terms $\widetilde{\theta}^{\nD}_t$ and $\hat{\theta}_t^{\nD}
$. In the weighted setting, $g_{t-1}: \mathbb{R}^d \mapsto \mathbb{R}^d$ is defined
 by $g_{t-1}(\theta) = \sum_{s=1}^{t-1} \gamma^{t-1-s} \mu(A_s^{\top} \theta) A_s +
  \lambda \theta $. The associated Jacobian matrix denoted by $J_{t-1}$ verifies
   $J_{t-1}(\theta) = \sum_{s=1}^{t-1} \gamma^{t-1-s} \dot{\mu}(A_s^{\top} \theta) A_s
    A_s^{\top}  + \lambda I_d$. $\hat{\theta}^{\nD}_t \, $  verifies $g_{t-1}(\hat{\theta}
    ^{\nD}_t) = \sum_{s=1}^{t-1} \gamma^{t-1-s} A_s X_s$.

We also need to introduce two more matrices,
\begin{equation}
\label{eq:V_t}
V_t = \gamma^{-t} W_t = \sum_{s=1}^t \gamma^{-s} A_s A_s^{\top} + \frac{\lambda \gamma^{-t}}{c_{\mu}} I_d
\end{equation}
and
\begin{equation}
\label{eq:V_tilde_t}
\widetilde{V}_t = \gamma^{-2t} \widetilde{W}_t = \sum_{s=1}^t \gamma^{-2s} A_s A_s^{\top} + \frac{\lambda \gamma^{-2t}}{c_{\mu}} I_d\;.
\end{equation}

In the previous equations $W_t$ and $\widetilde{W}_t$ are defined in Equation \eqref{eq_Design_matrix_D} and \eqref{eq_Design_matrix_D2} respectively. Thanks to the fundamental Theorem of Calculus with $G_{t-1}(\theta^{\star}_t, \tilde{\theta}_t) = \int_{0}^1 J_{t-1}(u \theta^{\star}_t + (1-u) \tilde{\theta}_t) du$, the following holds 
 \begin{equation}
 \label{eq_g_t_G_t_D}
 g_{t-1}(\theta^{\star}_t) - g_{t-1}(\tilde{\theta}_t) = G_{t-1}(\theta^{\star}_t, \tilde{\theta}_t) (\theta^{\star}_t- \tilde{\theta}_t).
 \end{equation}
 
Using the same argument than for Proposition \ref{prop_SW_GLM_concentration}, we have $G_t$ is an invertible matrix and 
$G_{t-1}(\theta^{\star}_t, \tilde{\theta}_t) \geq c_{\mu} W_{t-1}$. Knowing that $0<\gamma<1$, it ensures $\widetilde{W}_{t-1} \leq W_{t-1}$. Combining both inequalities gives,
\begin{equation}
\label{inequality_tilde_V_t_invert}
\widetilde{W}_{t-1} \leq W_{t-1} \leq \frac{1}{c_{\mu}} G_{t-1}(\theta^{\star}_t, \tilde{\theta}_t)\;,
\end{equation}
\begin{equation}
\label{inequality_G_t_invert}
 G_{t-1}^{-1}(\theta^{\star}_t, \tilde{\theta}_t) \leq \frac{1}{c_{\mu}} W_{t-1}^{-1}\;.
\end{equation}

We introduce the martingale $S_t = \sum_{s=1}^t \gamma^{-s} A_s \eta_s$.
Let $B_t = \sum_{s=1}^{t- D(\gamma)-1} \gamma^{-s} (\mu(A_s^{\top} \theta^{\star}_t) - \mu(A_s^{\top} \theta^{\star}_s) )A_s$ and let us abbreviate $G_t(\theta^{\star}_t, \tilde{\theta}_t)$ as $G_t$, then
  \begin{align*}
| \mu(\tilde{A}_t^{\top}
 \theta^{\star}_t) - \mu(\tilde{A}_t^{\top} 
\tilde{\theta}_t) |  
& \leq k_{\mu}|\tilde{A}_t^{\top}(\theta^{\star}_t-
 \tilde{\theta}_t) | 
 = k_{\mu} |\tilde{A}_t^{\top} G_{t-1}^{-1}
 (g_{t-1}(\theta^{\star}_t)-g_{t-1}(\tilde{\theta}_t))| \quad \textnormal{(Equation \eqref{eq_g_t_G_t_D})} \\
 &=k_{\mu} |\tilde{A}_t^{\top} G_{t-1}^{-1} \widetilde{W}_{t-1}^{1/2} \widetilde{W}_{t-1}^{-1/2}
 (g_{t-1}(\theta^{\star}_t)-g_{t-1}(\tilde{\theta}_t))| \\
 &\leq k_{\mu} \lVert \tilde{A}_t \rVert_{G_{t-1}^{-1} \widetilde{W}_{t-1} G_{t-1}^{-1}} \lVert g_{t-1}(\theta^{\star}_t)-g_{t-1}(\tilde{\theta}_t)\rVert_{ \widetilde{W}_{t-1} ^{-1}} \quad \textnormal{(C-S)}\\ 
& \leq \frac{k_{\mu}}{\sqrt{c_{\mu}}}  \lVert \tilde{A}_t \rVert_{G_{t-1}^{-1}} \lVert g_{t-1}(\theta^{\star}_t)-g_{t-1
}(\tilde{\theta}_t)\rVert_{ \widetilde{W}_{t-1} ^{-1}} \quad \textnormal{(Inequality \ref{inequality_tilde_V_t_invert})} \\
&\leq \frac{k_{\mu}}{c_{\mu}} \lVert \tilde{A}_t \rVert_{W_{t-1}^{-1}} \lVert g_{t-1}(\theta^{\star}_t)-g_{t-1}(\tilde{\theta}_t)\rVert_{ \widetilde{W}_{t-1} ^{-1}} \quad \textnormal{(Inequality \ref{inequality_G_t_invert})}\\
&  \leq 2\frac{k_{\mu}}{c_{\mu}} \lVert \tilde{A}_t \rVert_{W_{t-1}^{-1}} \lVert g_{t-1}(\theta^{\star}_t)-g_{t-1}(\hat{\theta}_t)\rVert_{\widetilde{W}_{t-1}^{-1}} \quad \textnormal{(Definition of } \tilde{\theta}_t \textnormal{)}\\
& \leq \frac{2k_{\mu}}{c_{\mu}} \lVert \tilde{A}_t \rVert_{W_{t-1}^{-1}} 
\left(  \lVert \sum_{s= 1}^{t-1} \gamma^{t-1-s} \mu(A_s^{\top} \theta^{\star}_t) A_s
   - \sum_{s= 1}^{t-1} \gamma^{t-1-s} A_s X_s  
  \rVert_{\widetilde{W}_{t-1}^{-1}}  +  \lVert \lambda \theta^{\star}_t \rVert_{\widetilde{W}_{t-1}^{-1}}\right)  \\
&  \leq \frac{2k_{\mu}}{c_{\mu}} \lVert \tilde{A}_t \rVert_{W_{t-1}^{-1}} 
\left(  \lVert \sum_{s= 1}^{t-1} \gamma^{-s} (\mu(A_s^{\top} \theta^{\star}_t) - \mu(A_s^{\top} \theta^{\star}_s) )A_s
  -\sum_{s= 1}^{t-1} \gamma^{-s}A_s \eta_s  \rVert_{\widetilde{V}_{t-1}^{-1}}  + \sqrt{\lambda c_{\mu}} S \right) \\
 &\leq \frac{2k_{\mu}}{c_{\mu}} \lVert \tilde{A}_t \rVert_{W_{t-1}^{-1}} 
\left(  \lVert B_t
  -\sum_{s= 1}^{t-1} \gamma^{-s}A_s \eta_s
  \rVert_{\widetilde{V}_{t-1}^{-1}}  + \sqrt{\lambda c_{\mu}} S \right) \quad \textnormal{(Thanks to $t \in \mathcal{T}(\gamma)$)}\\
  &\leq \frac{2k_{\mu}}{c_{\mu}}  \lVert \tilde{A}_t \rVert_{W_{t-1}^{-1}} 
 \left( \lVert B_t \rVert_{\widetilde{V}_{t-1}^{-1}}
  + \lVert S_{t-1} \rVert_{\widetilde{V}_{t-1}^{-1}}  
  + \sqrt{c_{\mu} \lambda} S \right) \quad \textnormal{(Triangle Inequality)}.
  \end{align*}

By using the results of Corollary \ref{corollary:S_t} and the fact that $(\eta_s)_{s \geq 1}$ are conditionally $\X/2$-subgaussian,
 with probability $\geq 1- \delta$ it holds that 
$$ 
\forall t \geq 1,  \lVert S_t \rVert_{\widetilde{V}_t^{-1}} \leq \frac{\X}{2} \sqrt{ 2\log\left(\frac{1}{\delta}\right) +d \log\left(1 +
\frac{ c_{\mu} L^2 (1-\gamma^{2t})}{\lambda d (1-\gamma^2)}\right)}\;.
$$
The next step consists in upper-bounding the bias term $B_t$.
\begin{align*}
\lVert B_t \rVert_{\widetilde{V}_{t-1}^{-1}} &= \lVert \sum_{s=1}^{t- D(\gamma)-1} \gamma^{-s} (\mu(A_s^{\top} \theta^{\star}_t) - 
\mu(A_s^{\top} \theta^{\star}_s) )A_s  \rVert_{\widetilde{V}_{t-1}^{-1}}  \\
&\leq \sqrt{\frac{c_{\mu}}{\lambda \gamma^{-2(t-1)}}}  \left\lVert \sum_{s=1}^{t- D(\gamma)-1} \gamma^{-s} (\mu(A_s^{\top} \theta^{\star}_t) - 
\mu(A_s^{\top} \theta^{\star}_s) )A_s  \right\rVert_2 
\quad (\widetilde{V}_{t-1} \geq \frac{\lambda \gamma^{-2(t-1)}}{c_{\mu}})\\
&\leq  \sqrt{\frac{c_{\mu}}{\lambda \gamma^{-2(t-1)}}}  \sum_{s=1}^{t- D(\gamma)-1} \gamma^{-s} |(\mu(A_s^{\top} \theta^{\star}_t) - 
\mu(A_s^{\top} \theta^{\star}_s))| \lVert A_s \rVert_2 \quad \textnormal{(Triangle Inequality)}\\
&\leq  L \sqrt{\frac{c_{\mu}}{\lambda} }  \sum_{s=1}^{t- D(\gamma)-1} \gamma^{t-1-s} |(\mu(A_s^{\top} \theta^{\star}_t) - 
\mu(A_s^{\top} \theta^{\star}_s))| \\
&\leq  L \sqrt{\frac{c_{\mu}}{\lambda} }  \sum_{s=1}^{t- D(\gamma)-1} \gamma^{t-1-s} k_{\mu} |A_s^{\top} (\theta^{\star}_t - \theta^{\star}_s)| \quad
\textnormal{(Assumption \ref{assumption_c_mu})}\\
&\leq  2 L^2 S  k_{\mu}  \sqrt{\frac{c_{\mu}}{\lambda} }  \sum_{s=1}^{t- D(\gamma)-1} \gamma^{t-1-s}
 \quad \textnormal{(C-S + Assumption \ref{assumption_actions} + Assumption \ref{assumption_param})} \\
  & \leq 2 L^2 S k_{\mu}  \sqrt{\frac{c_{\mu}}{\lambda}} \frac{\gamma^{D(\gamma)}}{1- \gamma}\;.
\end{align*}
The result is obtained by combining the inequalities.
\end{proof}

\subsection{Proof of Theorem \ref{theorem_regret_D}}
\label{subsection:regret_D}

\begin{customthm}{2}[Regret of $\D$]
The regret of the $\D$ policy is upper-bounded with probability  $\geq 1- 2\delta$ by
\begin{equation*}
R_T \leq  2\rho^{\nD}_T(\delta)\sqrt{2dT} \sqrt{T \log \left(\frac{1}{\gamma}\right)\hspace{-0.05cm}+\hspace{-0.05cm}\log\left(1 \hspace{-0.05cm}
+
\frac{c_{\mu} L^2}{d\lambda(1-\gamma)} \right)}  + \X \Gamma_T D(\gamma)\;,
\end{equation*}
where
$\rho^{\nD}$ is defined in Equation \eqref{eq_rho_D} and $\Gamma_T$ is the number of changes up to time $T$.
\end{customthm}
\begin{proof}
The regret is defined in the following way.
\begin{align*}
R_T
&= \sum_{t \notin \mathcal{T}(\gamma)}  (\mu(A_{t,\star}^{\top} \theta^{\star}_t) - \mu(A_{t}^{\top} \theta^{\star}_t)) 
+\sum_{t \in \mathcal{T}(\gamma)}  (\mu(A_{t,\star}^{\top} \theta^{\star}_t) - \mu(A_{t}^{\top} \theta^{\star}_t)) \\
&  \leq \X \Gamma_T  D(\gamma) + \sum_{t \in \mathcal{T}(\gamma)}  \min\{\X ,\mu(A_{t,\star}^{\top} \theta^{\star}_t) - \mu(A_{t}^{\top} \theta^{\star}_t)\}\;.
\end{align*}
By using the result of Corollary
 \ref{prop_D_GLM_anytime_upper}, it holds that with probability
  $\geq 1- 2\delta$
\begin{align*}
\begin{split}
R_T &\leq  \X \Gamma_T D(\gamma) +   \sum_{t \in \mathcal{T}(\gamma)} \min\{\X,2 \rho_t^{\nD}(\delta) \lVert A_t \rVert_{W_{t-1}^{-1}}\} \\
&\leq  \X \Gamma_T D(\gamma) +    2 \rho_T^{\nD}(\delta) \sum_{t \in \mathcal{T}(\gamma)} \min\{1,   \lVert A_t \rVert_{W_{t-1}^{-1}} \} \\
&\leq  \X \Gamma_T D(\gamma) +    2\rho_T^{\nD}(\delta) \sqrt{T} \sqrt{\sum_{t=1}^T  \min\{1, \lVert A_t \rVert_{W_{t-1}^{-1}}^2\}} \quad \textnormal{(C-S)}\;.
\end{split}
\end{align*}

Based on the proof of Proposition 4 in Appendix B of \cite{russac2019weighted}, we have
$$
\sqrt{\sum_{t=1}^T  \min\{1, \lVert A_t \rVert_{W_{t-1}^{-1}}^2\}} \leq \sqrt{2d} \sqrt{T \log\left(  \frac{1}{\gamma} \right) +
 \log\left(1 + \frac{c_{\mu}L^2}{d\lambda(1-\gamma)} \right) }\;.
$$
Therefore, with probability greater than $1-2\delta$,
$$
R_T \leq 2\rho_T^{\nD}(\delta) \sqrt{2d} \sqrt{T} \sqrt{T \log \left(\frac{1}{\gamma}\right) + \log\left(1 + \frac{c_{\mu}L^2}{d\lambda(1-\gamma)} \right)  } + \X \Gamma_T D(\gamma)\;.
$$ 
\end{proof}
\subsection{Proof of Corollary \ref{corollary:asympt_regret_D}}
\label{subsec:asympt_regret_D}
 \begin{customcor}{4}
 By taking $D(\gamma) = \frac{\log(1/(1-\gamma))}{1- \gamma}$,
 \begin{enumerate}

\item If $\Gamma_T$ is known, by choosing $\gamma = 1- (\frac{\Gamma_T}{dT})^{2/3}$, the regret of the $\D$ algorithm is asymptotically upper bounded
with high probability by a term $\tilde{O}(d^{2/3} \Gamma_T^{1/3} T^{2/3})$.

\item If $\Gamma_T$ is unknown, by choosing $\gamma= 1- \frac{1}{d^{2/3}T^{2/3}}$, the regret of the $\D$ algorithm is asymptotically upper bounded
with high probability by a term $\tilde{O}(d^{2/3} \Gamma_T T^{2/3})$.
 \end{enumerate}
\end{customcor}

\begin{proof}
Let $\gamma$ be defined as $\gamma = 1- (\frac{\Gamma_T}{dT})^{2/3}$ and $D(\gamma)= \frac{\log(1/(1-\gamma))}{(1-\gamma)}$.
With this choice of $\gamma$, $D(\gamma)$ is equivalent to $d^{2/3} \Gamma_T^{-2/3}T^{2/3} \log(T)$. Thus, $D(\gamma) \Gamma_T$
 is equivalent to $ d^{2/3} \Gamma_T^{1/3} T^{2/3} \log(T/\Gamma_T)$. 

In addition,
\begin{align*}
\gamma^{D(\gamma)} &= \exp(D(\gamma) \log(\gamma)) = \exp \left( - \frac{\log(\gamma)}{1- \gamma} \log(1-\gamma) \right)  \sim 1-\gamma\;.
\end{align*}

Hence, when omitting the logarithmic terms, $\rho_T^{\nD}(\delta)$ behaves as $\sqrt{d}$.

Furthermore, $\log(1/\gamma) \sim d^{-2/3} \Gamma_T^{2/3} T^{-2/3}$, implying that $T \log(1/\gamma) \sim d^{-2/3} \Gamma_T^{2/3} T^{1/3}$.

As a result, it holds that when neglecting the log terms,
$$\rho_T^{\nD}(\delta) \sqrt{dT} \sqrt{T \log(1 / \gamma) + \log\left( 1 + \frac{c_{\mu}L^2}{d \lambda (1- \gamma)}\right)} \approx d T^{1/2}  \sqrt{d^{-2/3} \Gamma_T^{2/3} T^{1/3}} 
= d^{2/3} \Gamma_T^{1/3} T^{2/3}\;.
$$

We obtain the desired result.
\end{proof}

\section{Subgaussianity of the noise term}

\subsection{Conditional Hoeffding lemma}
\label{subsec:hoeffding_cond}
\begin{lemma}[Conditional Hoeffding lemma]
\label{lemma:hoeffding_conditional}
Let $(\Omega, \mathcal{F}, (\mathcal{F}_t)_{t \geq 0}, \mathbb{P})$ be a probability space where $(\mathcal{F}_t)_{t \geq 0}$ is a filtration and $(X_t)_{t \geq 0}$ is a sequence of adapted random variables. Under the assumptions: 
\begin{enumerate}
\item $G_t$ is $(\mathcal{F}_{t-1})$-measurable
\item $G_t+ a_t \leq X_t \leq G_t + b_t$  \quad a.s
\item $\mathbb{E}\left[X_t | \mathcal{F}_{t-1} \right] = 0$
\end{enumerate}

Then,
$$
\forall \lambda \in \mathbb{R}, \mathbb{E} \left[  e^{\lambda X_t} | \mathcal{F}_{t-1} \right] \leq e^{\frac{\lambda^2 (b_t- a_t)^2}{8}}, \quad a.s. 
$$
\end{lemma}

This means that under the assumption of Lemma \ref{lemma:hoeffding_conditional}, $X_n$ is $(b_n-a_n)/2$-subgaussian conditionally on the past.

\subsection{Consequence on the noise term in GLMs}

In our bandit setting, the filtration associated with the random observations is denoted $\mathcal{F}_t = \sigma(X_1,...,X_t)$ and is such that $A_t$ is $\mathcal{F}_{t-1}$-measurable and $\eta_t$ is $\mathcal{F}_t$-measurable. Under assumption \ref{assumption_noise}, $\eta_t = X_t - \mu(A_t^{\top} \theta^{\star}_t)$ satisfies:

\begin{enumerate}
\item $- \mu(A_t^{\top} \theta^{\star}_t) \leq \eta_t \leq \X  -  \mu(A_t^{\top} \theta^{\star}_t)$ \quad a.s
\item $\mu(A_t^{\top}\theta^{\star}_t)$ is  $\mathcal{F}_{t-1}$-measurable
\item $ \mathbb{E} \left[   \eta_t | \mathcal{F}_{t-1}  \right] = 0 $
\end{enumerate}

Lemma 1 implies that  $\eta_t$ is  $m/2$-subgaussian conditionally on the past.

\end{document}